\documentclass[letterpaper, 10 pt, conference]{ieeeconf}  

\IEEEoverridecommandlockouts                              
\usepackage{subfigure}

\usepackage{setspace} 
\usepackage{colortbl}
\usepackage{graphicx}
\usepackage{epstopdf}
\usepackage{epsfig} 
\usepackage{mathptmx} 
\DeclareMathAlphabet{\mathcal}{OMS}{cmsy}{m}{n}
\usepackage{times} 
\usepackage{amsmath} 
\usepackage{amssymb}  
\usepackage{tabularx}
\usepackage{dcolumn}
\usepackage{color} 
\usepackage{algorithm} 
\usepackage{algorithmic}  
\usepackage[algo2e,linesnumbered,vlined,ruled]{algorithm2e} 
\usepackage{balance}
\usepackage{accents}
\usepackage{resizegather}
\usepackage{theorem} 
\usepackage{float}
\usepackage{mathtools}
\usepackage{multirow}
\usepackage{caption}

\usepackage[normalem]{ulem}
\usepackage[short]{optidef}
\usepackage{xcolor}
\usepackage{adjustbox}

\usepackage{setspace} 
\DeclareMathOperator{\ap}{\emph{A}\MRkern \emph{P}}
\newcommand{\MRkern}{%
  \mkern-3mu
  \mathchoice{}{}{\mkern0.2mu}{\mkern0.5mu}%
}

\newtheorem{definition}{\bf{Definition}}

\newtheorem{problem}{\bf{Problem}}

\newtheorem{theorem}{\bf{Theorem}}
\newtheorem{lemma}{Lemma}

\newcommand{\notltl}{\neg}
\newcommand{\andltl}{\wedge}
\newcommand{\orltl}{\vee}

\newcommand{\FP}{\mathcal{P}}
\newcommand{\FA}{\mathcal{A}}

\DeclareMathOperator*{\argmax}{arg\,max}
\title{\LARGE \bf

Reinforcement Learning Under Probabilistic Spatio-Temporal Constraints with Time Windows}

\author{Xiaoshan Lin$^{1*}$, Abbasali Koochakzadeh$^{2*}$, Yasin Yaz{\i}c{\i}o\u{g}lu$^3$, and Derya Aksaray$^4$ 
\thanks{$^*$These authors contributed equally to the paper.}
\thanks{This paper was supported by DARPA contract HR0011-21-2-0015.}
\thanks{$^1$X. Lin is with the Department of Aerospace Engineering and Mechanics, University of Minnesota, Minneapolis, MN, 55455, {\tt\small lin00668@umn.edu}.
}
\thanks{$^2$A. Koochakzadeh is with the Department of Electrical and Computer Engineering, Purdue University, West Lafayette, IN, 47907, {\tt\small akoochak@purdue.edu}}
\thanks{$^3$Y. Yaz{\i}c{\i}o\u{g}lu is with the Department of Mechanical and Industrial Engineering, Northeastern University, Boston, MA, 02115, {\tt\small y.yazicioglu@northeastern.edu}.}
\thanks{$^4$D. Aksaray is with the Department of Electrical and Computer Engineering, Northeastern University, Boston, MA, 02115, {\tt\small d.aksaray@northeastern.edu}}
}

\begin{document}
\maketitle

\begin{abstract}
We propose an automata-theoretic approach for reinforcement learning (RL) under complex spatio-temporal constraints with time windows. The problem is formulated using a Markov decision process under a bounded temporal logic constraint. Different from existing RL methods that can eventually learn optimal policies satisfying such constraints, our proposed approach enforces a desired probability of constraint satisfaction throughout learning. This is achieved by translating the bounded temporal logic constraint into a total automaton and avoiding ``unsafe" actions based on the available prior information regarding the transition probabilities, i.e., a pair of upper and lower bounds for each transition probability.  We provide theoretical guarantees on the resulting probability of constraint satisfaction. We also provide numerical results in a scenario where a robot explores the environment to discover high-reward regions while fulfilling some periodic pick-up and delivery tasks that are encoded as temporal logic constraints. 

\end{abstract}

\section{Introduction}

Reinforcement learning (RL) has been widely used to learn optimal policies for Markov decision processes (MDPs) through trial and error. Due to its capability to learn from interactions, adapt to changes in the environment, and handle different types of tasks in a general setting, RL has been widely adopted in the field of robotics. When the system is subject to constraints, traditional RL algorithms can learn a feasible policy by penalizing actions that lead to violating the constraints (e.g., \cite{49999}, \cite{sutton2018}). However, the lack of formal guarantees on constraint satisfaction during
the training phase can potentially cause undesirable outcomes, which limits the application of RL in many real-world scenarios. 

In recent years, safety in RL has been intensively studied. For example, control barrier function can be adopted to ensure safety during reinforcement learning by constraining the exploration set \cite{cheng2019} or augmenting the cost function \cite{marvi2021safe}. Safety in RL can be achieved by solving the constrained Markov decision process, where a cost function is used in a way that trajectories with expected costs smaller than a safety threshold will be considered safe (e.g., \cite{hasanzadezonuzy2021learning,Simo2021AlwaysSafeRL}, \cite{Altman1999ConstrainedMD, moldovan2012safe}). 
In conjunction with policy-gradient RL algorithms, the establishment of safety frameworks using the Hamilton-Jacobi reachability method is proposed in \cite{general-safety} to guarantee safety during learning for safety-critical systems. While many studies in safe RL ensure safety by avoiding undesirable states or actions, constraints beyond safety are not explored in the aforementioned works.

Temporal logic (TL) is a specification language that can describe rich behaviors, making it one of the most useful tools for specifying complex tasks for robots. There exist some works in the literature that explore constrained RL under TL constraints. For example, the works in \cite{hasan2020,cai2023safe} encode the desired constraints as Linear Temporal Logic (LTL) in a model-free learning framework and maximize the probability of satisfying LTL constraints. In \cite{alshiekh2018}, the risk of violating LTL constraints can trigger a reactive shield system that modifies the chosen action to ensure constraint satisfaction with maximum probability. Alternatively, Gaussian Process is adopted to estimate the uncertain system dynamics to ensure safety with high probability in \cite{cai2021safe}. In \cite{aksaray2021probabilistically}, the authors propose a constrained RL algorithm with probabilistic guarantees on the satisfaction of relaxed bounded TL constraints throughout learning. 


This paper addresses a  constrained RL problem as in \cite{aksaray2021probabilistically}. The objective is to learn a policy that ensures the satisfaction of a bounded TL constraint with a probability greater than a desired threshold in every episode. The main differences of this paper from \cite{aksaray2021probabilistically} are as follows: 1) The method in \cite{aksaray2021probabilistically} has guarantees on satisfying a temporal relaxation of the TL constraint. On the other hand, we provide guarantees on the satisfaction of TL constraint without any relaxation. 2) The method in \cite{aksaray2021probabilistically} is obtained by first deriving a closed-form lower bound on the probability of constraint satisfaction, which is applicable under a more restrictive setting in terms of the MDP, TL constraint, and available prior information. For instance, that method is not applicable under TL constraints such as ``globally stay in a specified region". In this paper, we significantly relax these limitations and, instead of a closed-form expression, we provide a recursive algorithm for computing a lower bound on the probability of constraint satisfaction in this generic setting.

\section{Preliminaries: Bounded Temporal Logic}
Temporal logic (TL) is a formal language to specify the temporal property of a system. Bounded TL such as Bounded Linear Temporal Logic (\cite{ishii2015}), Interval Temporal Logic \cite{10.1007/3-540-63010-4_6}, and Time Window Temporal Logic \cite{twtl} are expressive languages, as they can specify explicit time constraints (e.g., visit region $B$ in time-steps 0 to 2). Such bounded TL specifications can be translated to a deterministic finite state automaton (FSA).
\vspace{-2mm}
\begin{definition} [FSA]
A deterministic finite state automaton $\FA$ is a tuple $(Q, Q_{init}, \Sigma, \delta, F_{\FA})$ where
\begin{itemize}
\item $Q$ is a finite set of states;
\item $Q_{init} \subseteq Q$ is a set of initial states;
\item $\Sigma$ is a finite set of inputs;
\item $\delta : Q \times \Sigma \rightarrow Q$ represents the state transition relations;
\item $F_{\FA} \subseteq Q$ is a set of accepting states;
\end{itemize}
\end{definition}


Moreover, temporal relaxation of bounded TL can also be defined \cite{twtl}. A temporally relaxed formula is mainly the same formula structure with some extended or shrunk time windows. In that case, an FSA that encodes all temporal relaxations of a formula includes some backward edges and self-loops in non-accepting states. For example, consider a specification as ``Stay at region B for one time step within a time interval [0,2]" and its temporally relaxed version is ``Stay at region B for one time step within a time interval [0,2$+\tau$]". The illustrations of the original FSA and an FSA  encoding all temporal relaxations are shown in \ref{fig:FSA}. Note that the self-loops in the accepting states allow for the continuation of the mission after the satisfaction.


\begin{figure}[hb]
    \centering
    \hspace*{-4cm}
    \subfigure[FSA]{\includegraphics[width=0.38\columnwidth]{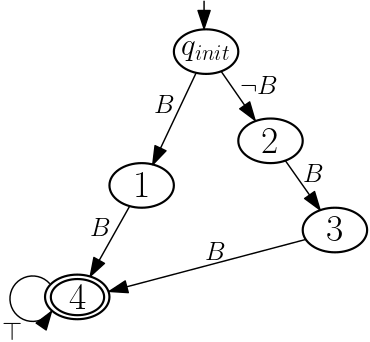}} 
    \hspace{1cm}
    \subfigure[Relaxed FSA]{\includegraphics[width=0.23\columnwidth]{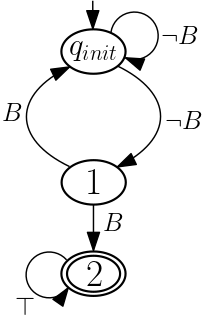}}
    \hspace*{-4cm}    
    \caption{(a) The FSA of specification $\phi:=$ ``Stay at region B for one time step within a time interval [0,2]". $Q=\{q_{init},1,2,3,4\}$, $\Sigma=\{B,\neg B\}$ where $\neg B$ refers to the negation of $B$, $\top$ is the true constant, $\delta(q_{init},B) = 1$, $\delta(q_{init},\neg B) = 2$, $\delta(1,B)=4$, $\delta(2,B)=3$, $\delta(3,B) =4$, $F_{\FA} = \{4\}$. (b) The relaxed FSA illustrates ``Stay at region B for one time step within a time interval [0,2$+\tau$]", for all feasible $\tau \in \mathbb{Z}$.}
    \label{fig:FSA}
\end{figure}
 \vspace{-8mm}
\begin{definition}
[Word, Accepting Word, Language] A word $\mathbf{o}=o(1) o(2) \dots$ is a sequence where $o(i)\in \Sigma$ for all $i \geq 1$. Any path over an FSA constitutes a word based on the respective edge labels. A path that starts from the initial state of the FSA and ends at an accepting state results in a satisfactory sequence, which is called an accepting word. The set of all accepting words is called the language of the corresponding FSA.
\end{definition}

Note that the original FSA only encodes the accepting words. In order to track the violation cases, one can also construct a total FSA.
\begin{definition}
[Total FSA] An FSA is called total if for all $q \in Q$ and any event $\sigma \in \Sigma$, the transition $\delta(q, \sigma) \neq \emptyset$ \cite{lin2015hybrid}.
\end{definition}
For any given FSA $\FA$, it is always possible to obtain a language-equivalent total FSA by introducing a trash state and adding a transition $\delta(q, \sigma)=\left\{trash\right\}$ if and only if $\delta(q, \sigma)=\emptyset$ in $\FA$. Fig.~\ref{fig:totalfsa} illustrates the total FSA for the specification ``Stay at region B for one time step within a time interval [0,2]". In Fig.~\ref{fig:totalfsa}, any path that starts from $q_{init}$ state and ends at state $4$ results in an accepting word; whereas any path that starts from $q_{init}$ state and ends at the trash state results in a rejecting word encoding a failure case.

\begin{figure}[htb!]
\begin{center}
\vspace*{-2mm}
\includegraphics[width=0.38\columnwidth]{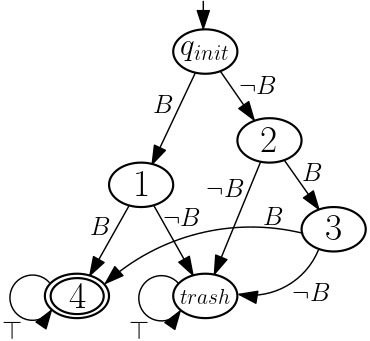}%
\caption{The total FSA $\FA=(Q,q_{init},\Sigma, \delta, F_{\FA}, \Psi)$ of the bounded TL specification ``Stay at region B for one time step within a time interval [0,2]". $Q=\{q_{init},1,2,3,4,trash\}$, $\Sigma=\{B,\neg B\}$ where $\neg B$ refers to the negation of $B$, $\top$ is the true constant, $\delta(q_{init},B) = 1$, $\delta(q_{init},\neg B) = 2$, $\delta(1,B)=4$, $\delta(1,\neg B)=trash$, $\delta(2,B)=3$, $\delta(2,\neg B)=trash$, $\delta(3,B) =4$,$\delta(3,\neg B)=trash$, $F_{\FA} = \{4\}$, and $\Psi=\{trash\}$.}
\label{fig:totalfsa}
\vspace{-4mm}
\end{center}
\end{figure}

\color{black}
\vspace{-3mm}
\section{Problem Statement}
We consider a Markov Decision Process (MDP) as $M=(S,A,\Delta_M,R)$, where $S$ is the set of states, $A$ is the set of actions, $\Delta_M:S \times A \times S \rightarrow [0,1]$ is the transition probability, and $R: S \times A \rightarrow \mathbb{R}$ is a reward function. We assume that every state $s \in S$ has a set of labels, and the labeling function $l:S \rightarrow 2^{\ap}$ maps every ${s \in S}$ to the power set of atomic propositions $\ap$. Fig.~\ref{fig:mdp} shows an example MDP with a set of atomic propositions and a labeling function. 
\begin{figure}[htb!]
\begin{center}
    \includegraphics[width=0.7\columnwidth, trim={0 0cm 0 0cm},clip]{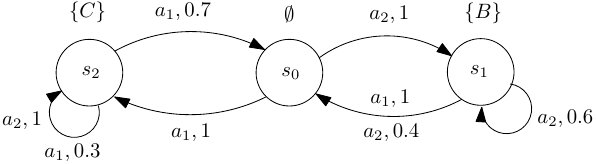}
\caption{An MDP where $S=\{s_0,s_1, s_2\}$, $A=\{a_1,a_2\}$, $\ap=\{B,C\}$, $l(s_{2})=\{C\}$, $l(s_0) = \emptyset$, $l(s_1)=\{B\}$. An edge from $s_i$ to $s_j$ under action $a_k$, i.e., $\Delta_M(s_i,a_k,s_j)$, induces the corresponding transition with the respective transition probability.  }
\label{fig:mdp}
\vspace{-4mm}
\end{center}
\end{figure}

Given an MDP, $\pi:S\rightarrow A$  is called a policy. In RL, the transition probability $\Delta_M$ is unknown, and the agent is required to find the optimal control policy $\pi^*$ that maximizes the expected sum of rewards, i.e., $E^{\pi}\Big[\sum_{t=0}^T \gamma^t r(t)\Big]\text{ or } E^{\pi}\Big[\sum_{t=0}^\infty \gamma^t r(t)\Big]$ where $r(t)$ is the reward collected at time $t$ and $\gamma \in (0,1]$ is a discount factor. In the literature, various learning algorithms, e.g., Q-learning \cite{watkins1992}, were shown to find the optimal policy.

In this paper, we address the problem of learning a policy such that a bounded TL specification is satisfied with a probability greater than a desired threshold throughout learning.  
 To provide a probabilistic satisfaction guarantee on the bounded TL constraint throughout the learning process, we assume some prior information about the transition probabilities of the MDP, even though the exact transition probabilities are unknown. We consider a generic setting where the prior information is encoded as lower and upper bounds on the transition probabilities. More specifically, for any two states $s,s' \in S$ and any action $a\in A$, let the probability of transitioning from $s$ to $s'$ under the action $a$ be bounded as:
\begin{equation}
\label{not_really_an_assumption}
\Delta_{min}(s,a,s') \leq \Delta_M(s,a,s') \leq \Delta_{max}(s,a,s'),
\end{equation}
where $\Delta_{min}:S \times A \times S \rightarrow [0,1]$ and $\Delta_{max}:S \times A \times S \rightarrow [0,1]$ are the given lower and upper bounds, respectively. 
The case where no prior information is available can be encoded in this setup as $\Delta_{min}(s,a,s')=0$ and $\Delta_{max}(s,a,s')=1$ for every tuple $(s,a,s')$. These intervals become narrower if more accurate information about the transition probability is provided\footnote{Such lower and upper bounds can come from prior information and domain knowledge, or be obtained by estimating the confidence interval of the transition probabilities \cite{calinescu2015formal, zhao2019probabilistic}. }, and ultimately $\Delta_{min}(s,a,s')=\Delta_{max}(s,a,s')$ if $\Delta_M(s,a,s')$ is known exactly. 


 
\vspace{-2mm}
\begin{problem} 
Suppose that the following are given:
\begin{itemize}
\item an MDP $M=(S,A,\Delta_M,R)$, where the transition probability $\Delta_M$ and reward function $R$ are unknown,
\item $\Delta_{min}:S \times A \times S \rightarrow [0,1]$ and $\Delta_{max}:S \times A \times S \rightarrow [0,1]$ such that \eqref{not_really_an_assumption} holds for every tuple $(s,a,s')$,
\item a set of atomic propositions $\ap$, 
\item a labeling function $l:S \rightarrow 2^\emph{$\ap$}$,
\item a bounded TL specification (constraint) $\phi$ 
\item a desired probability of satisfaction ${Pr_{des} \in (0,1]}$.
\end{itemize}
Learn the optimal control policy
\begin{equation}
\label{pistareq}
\pi^* = \arg\max\limits_{\pi } E^{\pi}\Big[\sum_{t=0}^\infty \gamma^t r(t) \Big] 
\end{equation}
such that, in each episode $m\geq1$ of the learning process,
\begin{equation}
\label{eq:Qobj}
\begin{aligned}
\Pr(\mathbf{o}_m,\phi)\geq Pr_{des},
\end{aligned}
\end{equation}
where $\mathbf{o}_m=o_m(1) o_m(2) \dots o_m(T)$ is a word based on the sequence of MDP states $s_m(t)$ observed in the $m^{th}$ episode, i.e., $o_m(t) = l(s_m(t))$, $T$ is the episode length determined by the time bound\footnote{The time bound of $\phi$ is the maximum time needed to satisfy it \cite{twtl}.} of $\phi$, $\Pr(\mathbf{o}_m,\phi)$ is the probability of satisfying $\phi$ by the word $\mathbf{o}_m$, and $\gamma \in (0,1)$ is the discount factor.  

\label{problem}
\end{problem}



\section{Proposed Approach}


Learning with TL constraints typically involves an MDP that encodes both the state information and progress towards the constraint satisfaction. 
A common approach is to construct a product MDP from the FSA encoding the TL constraint and the MDP encoding the state information. A similar approach was considered in \cite{aksaray2021probabilistically}, where a product MDP was constructed by using the relaxed FSA of the desired TL specification. While the use of relaxed FSA provides the benefit of a more compact representation (improving the scalability of learning), it also allows for learning control policies that result in trajectories satisfying not only the original specification but also its temporally relaxed version. In this paper, we focus on the problem of learning control policies that only generate trajectories ensuring the satisfaction of the original specification. To this end, we utilize total FSA instead of relaxed FSA. 

\vspace{-2mm}
\begin{definition} (Total Product MDP)
Given an MDP $M=(S,A,\Delta_M,R)$, a set of atomic propositions \emph{$\ap$}, a labeling function $l:S \rightarrow 2^\emph{$\ap$}$, and a total FSA $\FA=(Q, Q_{init}, 2^{\ap}, \delta, F_{\FA}, q_{trash})$, a product MDP is a tuple $\FP =M \times \FA = (S_{\FP},P_{init},A,\Delta_\FP, R_{\FP}, F_{\FP}, \Psi_{\FP})$, where

\begin{itemize}
\item $S_\FP = S \times Q$ \textit{is a finite set of states;}
\item $P_{init}=\{(s,\delta(q_{init},l(s)) \mid \forall s \in S \}$ is the set of initial states;
\item $A$ \textit{is the set of actions;}
\item $\Delta_\FP : S_\FP \times A \times S_\FP \rightarrow [0,1]$ \textit{is the probabilistic transition relation such that for any two states, $p=(s,q) \in S_\FP$ and $p^{\prime}=(s^{\prime},q^{\prime}) \in S_\FP$, and any action $a \in A$, $\Delta_\FP(p,a,p^\prime)=\Delta_M(s,a,s^\prime)$ and $\delta(q,l(s'))=q^\prime$}; 
\item $R_{\FP}:S_\FP \times A \rightarrow \mathbb{R}$ \textit{is the reward function such that $R_{\mathcal{P}}(p,a) = R(s,a)$ for $p=(s,q) \in S_\FP$;}
\item $F_\FP = (S \times F_\FA) \subseteq S_\FP$ \textit{is the set of accepting states.}
\item {$\Psi_{\FP} = (S \times q_{trash}) \subseteq S_\FP$ \textit{is the set of trash states.}}
\end{itemize}
\end{definition}


The knowledge of remaining time is also crucial in deciding actions under bounded time constraints. For example, suppose that a robot must visit region A before a specified deadline, and it has not visited region A yet. If there is a small amount of time left, the optimal action would be moving towards A. However, if there is more remaining time, the optimal action might be moving towards other regions that yield higher rewards. For this reason, we construct a time-total product MDP.    
\vspace{-2mm}
\begin{definition} (Time-Total Product MDP)
Given a product MPD $\mathcal{P}=(S_{\FP},P_{init},A,\Delta_\FP, R_{\FP}, F_{\FP}, \Psi_{\FP})$ and a time set $\mathcal{T}=\{0,\dots,T\}$, a time-product MPD is a tuple $\mathcal{P}^{\mathcal{T}} = \mathcal{P} \times \mathcal{T} =(S^{\mathcal{T}}_{\mathcal{P}}, P^{\mathcal{T}}_{init}, A, \Delta^{\mathcal{T}}_\FP, R^{\mathcal{T}}_{\mathcal{P}}, F^{\mathcal{T}}_{\mathcal{P}}, {\Psi}^{\mathcal{T}}_{\mathcal{P}})$ where
\begin{itemize}
\item $S^{\mathcal{T}}_{\mathcal{P}} = S_{\mathcal{P}} \times \mathcal{T}$ is a finite set of states;
\item $P^{\mathcal{T}}_{init} = P_{init} \times \{0\} \subseteq S^{\mathcal{T}}_{\mathcal{P}}$ is the set of initial states;
\item $A$ is the set of actions;
\item $\Delta_\FP^{\mathcal{T}} : S^{\mathcal{T}}_{\mathcal{P}} \times A \times S^{\mathcal{T}}_{\mathcal{P}} \mapsto [0,1]$ is the probabilistic transition relation such that $\Delta_\FP^{\mathcal{T}}(p_i^t,a,p_j^{t+1})=\Delta_\FP(p_i,a,p_j)$ for an action $a \in A$ and two time-total product MDP states $p_i^t=(p_i,t) \in S^{\mathcal{T}}_{\mathcal{P}}$ and $p_j^{t+1}=(p_j,t+1) \in S^{\mathcal{T}}_{\mathcal{P}}$;
\item $R^{\mathcal{T}}_{\mathcal{P}}  : S^{\mathcal{T}}_{\mathcal{P}} \times A \mapsto \mathbb{R}$ is the reward function such that $R^{\mathcal{T}}_{\mathcal{P}}(p^t ,a)=R_{\mathcal{P}}(p,a)$ and $p^t=(p,t) \in S^{\mathcal{T}}_{\mathcal{P}}$;
\item $F^{\mathcal{T}}_{\mathcal{P}} = (F_{\mathcal{P}} \times \mathcal{T}) \subseteq S^{\mathcal{T}}_{\mathcal{P}}$ is the set of accepting states.
\item ${\Psi}^{\mathcal{T}}_{\mathcal{P}} = (\Psi_{\mathcal{P}} \times \mathcal{T}) \subseteq S^{\mathcal{T}}_{\mathcal{P}}$ is the set of trash states.
\end{itemize}
\label{def:tmdp}
\end{definition}

\vspace{-2mm}
Note that the time set in constructing the time-total product MDP is selected according to the time bound of the constraint $\phi$ as defined in Problem~\ref{problem}. A time-total product MDP state with a time element $T$ should be either an \emph{accepting} or a \emph{trash} state. Moreover, the lower and upper bounds of the transition probabilities as shown in \eqref{not_really_an_assumption} can be projected on the time-total product MDP as follows:

\vspace{-3mm}

\begin{equation}
\label{not_really_an_assumption_tpmdp}
\Delta_{min}(p^{t}_i,a,p_j^{t+1}) \leq 
\Delta_\FP^{\mathcal{T}}(p^{t}_i,a,p_j^{t+1})
\leq \Delta_{max}(p_i^{t},a,p_j^{t+1}),
\end{equation}

where 
\begin{equation}
\small
\begin{split}
  & \Delta_{min}(p^{t}_i,a,p_j^{t+1}) =
    \begin{cases}
      \Delta_{min}(s_i,a,s_j) & \text{if $\delta(q_i,l(s_j))=q_j$ }\\
      0 & \text{else}
    \end{cases}  \\
  & \Delta_{max}(p^{t}_i,a,p_j^{t+1}) =
    \begin{cases}
      \Delta_{max}(s_i,a,s_j) & \text{if  $\delta(q_i,l(s_j))=q_j$ }\\
      0 & \text{else}
    \end{cases}
\end{split}
\end{equation}

\subsection{Evaluation of the Minimum Probability of Satisfaction}  
Given a time-total product MDP and lower/upper bounds of transition probabilities as in \eqref{not_really_an_assumption_tpmdp}, we propose an algorithm based on backward recursion to compute a lower bound on the probability of constraint satisfaction (i.e., the probability of reaching the set of $accepting$ states in the remaining time) for each state $p_i^t \in S^{\mathcal{T}}_{\mathcal{P}}$, denoted as $f(p_i^t) \in [0,1].$ Since any time-total product MDP states at $t=T$ is either a $trash$ or $accepting$ state, we can set the lower bound on constraint satisfaction for all those states as:
\begin{equation}\label{v(s)_T}
f(p_i^{T}) = \left\{\begin{array}{ll} \mbox{1, if $p_i^{T} \in F^{\mathcal{T}}_{\mathcal{P}}$,} \\ \mbox{0, if $p_i^{T} \in {\Psi}^{\mathcal{T}}_{\mathcal{P}}$.}\end{array}\right.
\end{equation}
Staring with these values, the proposed algorithm computes $f(p_i^{t})$ recursively for all $0\leq t \leq T-1$ as follows. For any action $a \in A$, let 
\begin{equation}
\label{N_definition}    
N(p_i^{t},a) = \{p_j^{t+1} \in S^{\mathcal{T}}_{\mathcal{P}} \mid \Delta_{max}(p_i^{t},a,p_j^{t+1})>0\}
\end{equation}
denote the set of states that may be reached from $p_i^{t}$ by taking action $a$ and let $n=|N(p_i^{t},a)|$. Then, for each state-action pair $(p_i^{t},a)$, we can obtain a lower bound on the probability of reaching an accepting state from $p_i^{t}$, conditioned on first taking action $a$. We obtain this lower bound by considering the ``worst-case", i.e., transitioning to states in $N(p_i^{t},a)$ with the smallest value of $f(p_i^{t+1})$ with the maximum possible probability.  More specifically, we formulate the following linear optimization problem to solve for a lower bound on the probability of reaching the accepting states from state $p_i^{t}$ by first taking action $a$:
\footnotesize
 \begin{mini!}|s|[2]  
    {\hat{\Delta}_{1},\hat{\Delta}_{2},\hdots,\hat{\Delta}_{n}}                               
    {\sum_{j=1}^{n} f(p^{t+1}_{j}) \hat{\Delta}_{j} \label{eq:eq1}}   
    {\label{opt_problem}}  
    {\kappa(p_i^{t},a) =}                               
    \addConstraint{\sum_{j=1}^{n} \hat{\Delta}_{j}}{=1 \label{eq:con1}}    
    \addConstraint{ \Delta_{min}(p_i^{t},a,p^{t+1}_j) \leq \hat{\Delta}_{j} }{\leq \Delta_{max}(p_i^{t},a,p^{t+1}_j), \label{eq:con2}}  
\end{mini!}
\normalsize

\noindent where $\Delta_{min}(p_i^{t},a,p^{t+1}_j)$ and $\Delta_{max}(p_i^{t},a,p^{t+1}_j)$ are lower and upper bound transition probabilities as in \eqref{not_really_an_assumption_tpmdp}.  Finally, for each state $p_i^{t}$ we define $f(p_i^{t})$ as
\begin{equation}
\label{opt_max_kappa}
f(p_i^{t}) = \max_{a\in A} \kappa(p_i^{t},a), 
\end{equation}
where $f(p_i^{t})$ denotes the largest lower bound on the probability of constraint satisfaction that can be achieved by properly choosing the action to take at $p_i^{t}$. 
 Note that once the values $f(\cdot)$ for states at time step $t$ are given, values $f(\cdot)$ for states at time step $t-1$ can be obtained by solving the optimization problem in \eqref{opt_problem} and using \eqref{opt_max_kappa}.

\subsection{Reinforcement Learning with Constraint Satisfaction}
We first define a policy $\pi_C$ that drives the agent towards the set of accepting states $F_{\FP}^{\mathcal{T}}$ by choosing the actions that yield the maximum lower bound $\kappa(\cdot)$ for each state. 
\vspace{-2mm}
\begin{definition}
[Go-to-$F_{\FP}^{\mathcal{T}}$ Policy] Given a time-total product MDP with the lower/upper bounds on transition probabilities as in \eqref{not_really_an_assumption_tpmdp}, Go-to-$F_{\FP}^{\mathcal{T}}$ policy $\pi_{C}: S_{\FP}^{\mathcal{T}} \rightarrow A $ is defined as
\begin{equation}
\label{policy_equation}  
\pi_{C}(p_i^{t})= \argmax \limits_{a \in A}\kappa(p_i^{t},a), \; \forall p_i^t\in S_{\FP}^{\mathcal{T}}.
\end{equation}

\end{definition}

\begin{lemma} Given a time-total product MDP $\FP^{\mathcal{T}}$ with known lower and upper bounds defined in \eqref{not_really_an_assumption_tpmdp} and time bound $T$, for any $p_i^t \in  S_{\FP}^{\mathcal{T}}$, let ${\Pr\big(p_i^t \rightarrow F_\FP^{\mathcal{T}};\pi_{C} \big)}$ be the probability of reaching the set of accepting states $F_{\FP}^\mathcal{T} \subseteq S_{\FP}^\mathcal{T}$ from $p_i^t$ in the next $T-t$ time steps under the policy $\pi_{C}$. Then
\begin{equation}
\label{lower-bound}
\Pr\big(p_i^t \rightarrow F_\FP^{\mathcal{T}};\pi_{C} \big) \geq f(p_i^t),
\end{equation}
where $f(p_i^t)$ is derived from \eqref{opt_max_kappa}.
\label{lemma}
\end{lemma}

\begin{proof}
We will prove the lemma by induction. We first show that \eqref{lower-bound} holds for step $t=T$. Since the states in $\FP^{\mathcal{T}}$ at $t=T$ are either $trash$ or $accepting$, under any policy the probability of reaching $F_{\FP}^\mathcal{T}$ from any of those states is  1 ($accepting$) or 0 ($trash$) by definition. Hence,
\begin{equation}
\label{induction_first_step}
{\Pr\big(p_i^T \xrightarrow{} F_\FP^{\mathcal{T}};\pi_{C} \big)} = \left\{\begin{array}{ll} \mbox{1, if $p_i^T \in F^{\mathcal{T}}_{\mathcal{P}}$,} \\ \mbox{0, if $p_i^T \in \Psi^{\mathcal{T}}_{\mathcal{P}}$.}\end{array}\right.
\end{equation}
By \eqref{v(s)_T} and \eqref{induction_first_step}, it can be concluded that \eqref{lower-bound} holds for $t=T$. 

Now suppose that \eqref{lower-bound} holds for all states at $t=k+1$ for some $0\leq k\leq T-1$. Then, we will show that \eqref{lower-bound} also holds for all states at $t=k$. Let $p_i^{k}$ denote a state in $\FP^{\mathcal{T}}$ at $t=k$. Let $a^*= \pi_{C}(p_i^{k})$ be the action at $p_i^{k}$ given by the policy $\pi_{C}$. For any state $p^{k+1}_{j} \in N(p_i^{k},a^*)$, recall that $\Delta_\FP^{\mathcal{T}}(p_i^{k},a^*,p^{k+1}_j)$ denotes the probability of reaching $p^{k+1}_{j}$ from $p_i^{k}$ by taking action $a^*$. Accordingly,
\begin{equation}
\footnotesize
\label{lemma_eq2}
\begin{split}
\Pr\big(p_i^{k} \rightarrow F_\FP^{\mathcal{T}};\pi_{C} \big) 
    & = \sum_{j=1}^{n} {\Pr\big(p_j^{k+1} \rightarrow F_\FP^{\mathcal{T}};\pi_{C} \big)} \Delta_\FP^{\mathcal{T}}(p_i^{k},a^*,p_j^{k+1}) \\
    & \geq \sum_{j=1}^{n} {f(p_j^{k+1})} \Delta_\FP^{\mathcal{T}}(p_i^{k},a^*,p_j^{k+1}),  
\end{split}
\end{equation}
where $n=|N(p_i^{k},a^*)|$ and the inequality is obtained from the premise that  \eqref{lower-bound} holds for every state at $t=k+1$. Furthermore, due to \eqref{opt_problem}, \eqref{opt_max_kappa}, and \eqref{policy_equation},

\begin{equation}\label{lemma_eq3}
       \sum_{j=1}^{n} f(p_j^{k+1}) \Delta_\FP^{\mathcal{T}}(p_i^{k},a^*,p_j^{k+1})  
        \geq  \kappa(p_i^{k},a^*) = f(p_i^{k}) .
\end{equation}

Due to \eqref{lemma_eq2} and \eqref{lemma_eq3}, \eqref{lower-bound} holds for $t = k$. Hence, we conclude by induction that \eqref{lower-bound} holds for any $t$ in $\{0,1,\dots,T\}$.
\end{proof}


Next, we present a modified Q-learning algorithm that learns a policy while the desired TL constraint is satisfied with a probability greater than $Pr_{des}$ in every episode. 
In a nutshell, the algorithm first prunes actions that may lead to states with a lower bound $f(\cdot)$ smaller than $Pr_{des}$ and only allows the remaining actions during learning. In cases where all the actions are pruned this way, an action is selected to maximize $\kappa(\cdot)$. In light of Lemma~\ref{lemma}, we propose the so-called one-shot pruning algorithm (Alg.~1a) to construct a pruned time-total product MDP. In Alg.~1b, we present a modified Q-learning algorithm to learn a policy over the pruned time-total product MDP.

 \begin{algorithm}[htb!]
 \label{alg:alg1}
 \begin{center}
\resizebox{\columnwidth}{!}{
\begin{tabular}{ll}
\bf{Alg. 1a:} \textbf{One-shot Pruning Algorithm for Time-total Product MDP} \\
\hline
 \emph{Input:} $ \mathcal{P}^{\mathcal{T}} = (S^{\mathcal{T}}_{\mathcal{P}}, P^{\mathcal{T}}_{init}, A, \Delta_\FP^{\mathcal{T}}, R^{\mathcal{T}}_{\mathcal{P}}, F^{\mathcal{T}}_{\mathcal{P}}, \Psi^{\mathcal{T}}_{\mathcal{P}})$ (time-total product MDP)\\
 \emph{Input:} $Pr_{des}$ (desired satisfaction probability)\\
\emph{Output:} $\mathcal{P}^{\mathcal{T}}$ (pruned time-product MDP), $\pi_C$\\
\hline 
\mbox{\small $\;1:\;$}\textbf{Initialization:}  $T_s= T,\,\,T_e=0,\,\,$ $Act(p^t)=A$ for all $p^t \in S^\mathcal{T}_\mathcal{P}$ \\
\hspace{2.5cm}  $f(p^{T_s}) \leftarrow \eqref{v(s)_T}$ for all $p^T \in S^\mathcal{T}_\mathcal{P}$ \\
\mbox{\small $\;2:\;$}\hspace{0.1cm}\textbf{for} each $t \in\{T_s-1,T_s-2,\dots,T_e\}$ \\
\mbox{\small $\;3:\;$}\hspace{0.5cm}\textbf{for} each $p^t$ s.t. $p^t \notin F^{\mathcal{T}}_{\mathcal{P}}$ and $p^t \notin \Psi^{\mathcal{T}}_{\mathcal{P}}$\\
\mbox{\small $\;4:\;$}\hspace{0.9cm}\textbf{for} each action $a \in Act(p^t)$\\
\mbox{\small $\;5:\;$}\hspace{1.3cm}$ N(p^{t},a) \leftarrow \eqref{N_definition}$; \\
\mbox{\small $\;6:\;$}\hspace{1.3cm}\textbf{if } $\exists p^{t+1}\in N(p^{t},a)$ s.t. $f(p^{t+1}) < Pr_{des}$ \\
\mbox{\small $\;7:\;$}\hspace{1.55cm}$Act(p^t)=Act(p^t)\setminus \{a\}$ ;\\
\mbox{\small $\;8:\;$}\hspace{1.3cm}\textbf{end if}\\
\mbox{\small $\;9:\;$}\hspace{1.3cm}$ \kappa(p^{t},a) \leftarrow \eqref{opt_problem}$; \\
\mbox{\small $\;10:\;$}\hspace{0.75cm}\textbf{end for}\\
\mbox{\small $\;11:\;$}\hspace{0.75cm}$f(p^t) \leftarrow $\eqref{opt_max_kappa};\\
\mbox{\small $\;12:\;$}\hspace{0.75cm}$ \pi_{C}(p^t)= \argmax \limits_{a}\kappa(p^{t},a) $;\\
\mbox{\small $\;13:\;$}\hspace{0.4cm}\textbf{end for}\\
\mbox{\small $\;14:\;$}\hspace{0.1cm}\textbf{end for}\\
\mbox{\small $\;15:\;$} $\mathcal{P}^{\mathcal{T}} =(S^{\mathcal{T}}_{\mathcal{P}}, P^{\mathcal{T}}_{init}, Act:S_\FP^\mathcal{T} \rightarrow 2^A, \Delta_\FP^{\mathcal{T}}, R^{\mathcal{T}}_{\mathcal{P}}, F^{\mathcal{T}}_{\mathcal{P}}, \Psi^{\mathcal{T}}_{\mathcal{P}})$
\end{tabular}}
\end{center}
\end{algorithm}

 \begin{algorithm}[htb!]
 \label{alg:alg1}
 \begin{center}
\resizebox{\columnwidth}{!}{
\begin{tabular}{ll}
\bf{Alg. 1b:} \textbf{Q-Learning with Constraint Satisfaction Guarantee} \\
\hline
 \emph{Input:} $\mathcal{P}^{\mathcal{T}} =(S^{\mathcal{T}}_{\mathcal{P}}, P^{\mathcal{T}}_{init}, Act:S_\FP^\mathcal{T} \rightarrow 2^A, \Delta_\FP^{\mathcal{T}}, R^{\mathcal{T}}_{\mathcal{P}}, F^{\mathcal{T}}_{\mathcal{P}}, \Psi^{\mathcal{T}}_{\mathcal{P}})$ \\
\emph{Input:} $p_{init}=(s_{init},q_{init},0) \in P_{init}^{\mathcal{T}}$, $\pi_C$ \\
\emph{Output:}  $\pi:S^\mathcal{T}_\mathcal{P} \rightarrow A$ \\
\hline 
\mbox{\small $\;1:\;$}\textbf{Initialization:} Initial $Q-$table, $p \leftarrow p_{init}$;  \\
\mbox{\small $\;2:\;$}$\text{flag}^{\pi_C} \leftarrow $ False; \\
\mbox{\small $\;3:\;$}\hspace{0.05cm}\textbf{for}\hspace{0.1cm}  $j=0:N_{episode}$ \\
\mbox{\small $\;4:\;$}\hspace{0.5cm}\textbf{for} $t=0:T$ \\ 
\mbox{\small $\;5:\;$}\hspace{0.9cm}\textbf{if} $\text{flag}^{\pi_C}$ or $Act(p)=\emptyset$\\
\mbox{\small $\;6:\;$}\hspace{1.25cm} Select an action \emph{a} from $\pi_C(p)$;\\
\mbox{\small $\;7:\;$}\hspace{1.25cm} $\text{flag}^{\pi_C} \leftarrow $ True;\\
\mbox{\small $\;8:\;$}\hspace{0.9cm}\textbf{else}\\
\mbox{\small $\;9:\;$}\hspace{1.4cm}Select an action \emph{a} from $Act(p)$ via $\epsilon-$greedy ($\pi$ policy);\\
\mbox{\small $\;11:\;$}\hspace{0.75cm}Take action \emph{a}, observe the next state $p^\prime=(s^\prime,q^\prime, t+1)$ and reward \emph{r};\\
\mbox{\small $\;12:\;$}\hspace{0.75cm}$Q(p,a) = (1-\alpha_{ep}) Q(p,a) + \alpha_{ep} \big[ r + \gamma \max\limits_{a^\prime}  Q(p^\prime,a^{\prime}) \big]$; \\
\mbox{\small $\;13:\;$}\hspace{0.75cm}$\pi(p) = \arg\max\limits_a Q(p,a))$; \\
\mbox{\small $\;14:\;$}\hspace{0.75cm}$p = p^\prime$; \\
\mbox{\small $\;15:\;$}\hspace{0.75cm}\textbf{if} $p$ is $Accepting$ or $Trash$\\
\mbox{\small $\;16:\;$}\hspace{1.1cm} $\text{flag}^{\pi_C} \leftarrow $ False;\\
\mbox{\small $\;17:\;$}\hspace{0.35cm}\textbf{end for}\\
\mbox{\small $\;18:\;$}\hspace{0.35cm}$ p = (s^\prime,q_{init},0)$; \\
\mbox{\small $\;19:\;$}\textbf{end for} 
\end{tabular}}
\end{center}
\end{algorithm}

Algorithm~1a is executed offline and its inputs are the desired probability of satisfaction $Pr_{des}$ and the time-total product MDP constructed from the MDP and the TL constraint $\phi$. The algorithm starts with initializing the time horizon $T$ as the time-bound and feasible action set $Act(\cdot)$ for each time-total product MDP state as $A$, the action set of the MDP. 
For each time-total product MDP state $p^t$ that is neither a trash state nor an accepting state, the pruning is done by inspecting every action. If an action $a$ can cause the system to move to a state $p^{t+1}$ where $f(p^{t+1})$ is less than $Pr_{des}$, then the action is pruned from $Act(p^t)$ (lines 6-8). Then, $f(p^t)$ is computed based on the values of $f(p^{t+1})$, and the action $\pi_{C}(p^t)$ is obtained. The output of the algorithm is a time-total product MDP with the pruned action sets $Act(\cdot)$ and the policy $\pi_C$.

A modified $Q$-learning algorithm for learning a policy on the pruned time-total product MDP is presented in Alg.~1b. For each time-total product MDP state $p$, if the feasible action set $Act(p)$ is not empty, then an action is selected from $Act(p)$ (line 9). If $Act(p)$ is empty, the RL agent will follow the policy $\pi_{C}$ until it reaches accepting states or trash states (lines 15-16). The general steps for the $Q$-updates are achieved in lines 11-13. 

The following theorem shows that the probability of reaching an accepting state (i.e., satisfying the desired TL constraint)  while executing Alg.~1 is at least $Pr_{des}$ in every episode of learning. 

\vspace{-2mm}
\begin{theorem}
\label{theorem1}
Given a time-total product MDP and lower/upper bounds of transition probabilities $\Delta_{min}$ and $\Delta_{max}$ as in \eqref{not_really_an_assumption}, 
let $f(\cdot)$ denote the lower bound function computed in Alg. 1a. If the set of initial states of the time-total product MDP, i.e., $P_{init}^\mathcal{T}$, satisfies
\begin{equation}
    \label{eq:th1}
    f(p_i^0) \geq Pr_{des}, \quad \forall p_i^0 \in P_{init}^\mathcal{T},
\end{equation}
then the probability of reaching the accepting states is at least $Pr_{des}$ in every episode while Alg. 1b is executed.
\end{theorem}

\begin{proof}
In each episode of learning, there are $T$ actions to be taken. Each action is either from the set $Act$ (line 9 Alg. 1b) or from policy $\pi_C$ (line 6 Alg. 1b). Line 5 in Alg. 1b implies that the agent might switch from $Act$ to $\pi_C$ at some point before reaching accepting states or trash states. Once the agent adopts $\pi_C$, it will follow it until reaching either the accepting states or trash states (Lines 7, 15-16). Therefore, there are three different cases of action sequences that could be observed in each episode in Alg. 1b:
\begin{itemize}
    \item \textbf{Case(a)}: Sequences such that all actions from time 0 to $T-1$ are taken from set $Act$.
    \item \textbf{Case(b)}: Sequences that start with $Act$ actions and switch to $\pi_C$ at some point until reaching a $trash$ or $accepting$ states. 
    \item \textbf{Case(c)}: Sequences that start with $\pi_C$ until reaching a $trash$ or $accepting$ state. 
\end{itemize}
Since these cases are disjoint,
we prove the theorem by showing that the probability of reaching an accepting state at time $t$ is at least $Pr_{des}$ in each of these cases.

\noindent\textbf{Case(a)} Remember that every state at $t=T$ is either $trash$ ($f(p_i^T)=0$) or $accepting$ ($f(p_i^T)=1$). Hence, for any $Pr_{des} \in (0,1]$, the only actions left in $Act(p_i^{T-1})$ after the pruning in lines 6-8 of Alg. 1a are those that surely lead to an accepting state, i.e., $N(p_i^{T-1},a)\subseteq F^{\mathcal{T}}_{\mathcal{P}}$ for every $a\in Act(p_i^{T-1})$. Thus,  the system surely reaches an accepting state in this case.


\noindent \textbf{Case(b)} For any such sequence of actions, let $t^{\prime} \in\{0,1,2, \ldots, T-2\}$ be the instant such that $a \in Act(p_i^{t'})$ and all the following actions are taken according to the policy $\pi_{C}$ until reaching a $trash$ or $accepting$ state. Due to the pruning in lines 6-8 of Alg. 1a, taking $a \in Act(p_i^{t'})$ surely leads to a state $p_i^{t'+1}$ such that ${\Pr\big(p_i^{t'+1} \xrightarrow{} F_\FP^{\mathcal{T}}; \pi_{C}\big)} \geq f(p_i^{t'+1}) \geq Pr_{des}$. Given that the system switches to $\pi_C$ starting at $p_i^{t'+1}$, we conclude that the probability of reaching an accepting state in this case is also at least $Pr_{des}$.

\noindent \textbf{Case(c)} In this case, each action is taken according to the policy $\pi_{C}$ for all time steps from 0 to $T-1$. By Lemma~\ref{lemma} and \eqref{eq:th1},
\begin{equation*}
    \Pr\big(p_i^{0} \rightarrow F_\FP^{\mathcal{T}};\pi_{C} \big) \geq f(p_i^0) \geq Pr_{des}, \quad \forall p_i^0 \in P_{init}^\mathcal{T},
\end{equation*} 
implying that the probability of reaching $accepting$ state (i.e. satisfying the constraint) is at least $Pr_{des}$. 
\end{proof}

\subsection{Multi-shot Pruning Algorithm}
While the proposed algorithm can guarantee probabilistic constraint satisfaction in every episode during RL, the learned policy might be overly-conservative since the agent will follow $\pi_C$ until it achieves the constraint satisfaction once it switches to $\pi_C$ at any point. This behavior eliminates the possibility of switching between making progress towards constraint satisfaction and collecting rewards as needed. To address this issue, we also propose a multi-shot pruning algorithm that allows the agent to switch between $\pi_C$ and exploration more effectively. The multi-shot pruning algorithm essentially decomposes the time-total product MDP into several subgraphs and applies the one-shot algorithm in each subgraph to ensure constraint satisfaction.


In the example shown in Fig.~\ref{fig:multi-layer}, a time-total product MDP $\mathcal{P}^{\mathcal{T}}$ is decomposed into three subgraphs, $G_1$(green), $G_2$(yellow) and $G_3$(red), given two time stamps $t_1$ and $t_2$. Let $Pr_1$, $Pr_2$, and $Pr_3$ be the desired probability thresholds for these subgraphs such that $\Pi_{i=1}^3 Pr_i = Pr_{des}$. Starting from the last sub-graph, the backward propagation is executed until the first layer of that sub-graph (instead of the initial layer of the time-total product MDP as depicted in the previous section). After calculating the worst-case maximum probability values, the states at the last layer of each sub-graph can be classified as accepting and trash. Overall, the backward propagation procedure that is described in the previous section is adopted to the sub-graphs of the time-total product MDP. The decomposition of the time-total product MDP into subgraphs is not unique. As we will show later in Theorem \ref{theorem2}, the probability of satisfaction can be ensured as long as the provided set of time and desired probabilities comply with the required conditions. 

\begin{figure}[ht]
 \begin{center}
\includegraphics[width=\columnwidth]{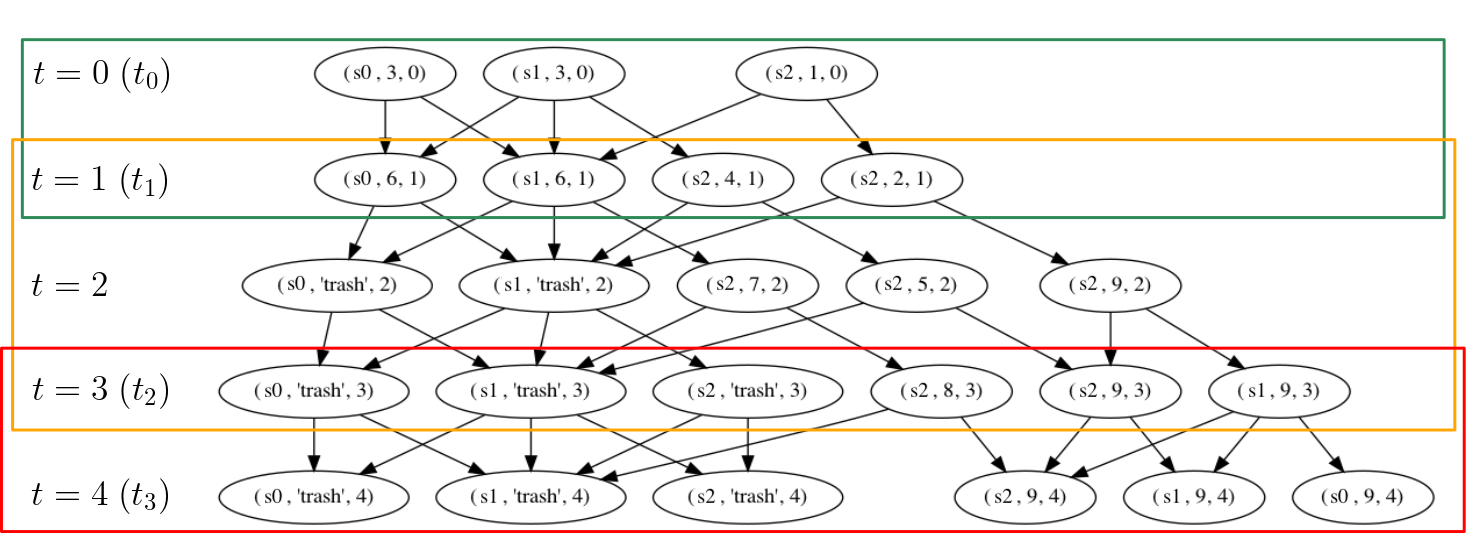}
 \caption{Partitioned time-total product MDP}
 \label{fig:multi-layer}
 \end{center}
\end{figure}
\vspace{-2mm}
We generalize Algs.~1a and 1b as a multi-shot pruning algorithm in Alg.~2a and the corresponding modified Q-learning algorithm in Alg.~2b. The inputs of the multi-shot pruning algorithm are a time-total product MDP, a set of time stamps $\{t_0=0, t_1, t_2, ..., t_{N-1}, t_N = T\}$ for decomposing the time-total product MDP, and a set of desired probabilities $\{Pr_1, Pr_2, ..., Pr_N\}$ such that $\prod\limits_{i=1}^{N}Pr_i = Pr_{des}$. In line 1, the time-total product MDP is divided into N subgraphs $\{G_i\}_{i=1,2,...N}$ (see Fig.~\ref{fig:multi-layer}), where each subgraph $G_i$ spans from $t_{i-1}$ to $t_i$. The one-shot pruning algorithm with a probability threshold $Pr_N$ is applied on the last subgraph $G_N$ to obtain the lower bound function $f^N$, pruned action set $Act^N$, and $\pi_C^N$ policy for $G_N$ (lines 2-3). We denote the sets of accepting states and trash states of time-total product MDP as $A_N$ and $T_N$ respectively (line 4). For each subgraph $G_i (i=N-1, N-2, ..., 1)$, we define the set of accepting states $A_i$ and set of trash states $T_i$ based on the function $f^{i+1}$ from the next subgraph $G_{i+1}$ (lines 6-7).  We then repeatedly apply the one-shot pruning algorithm with a probability threshold $Pr_i$ to obtain the function $f^i$, pruned action set $Act^i$, and $\pi_C^i$ policy for $G_i$ (lines 8-10). For each subgraph $G_i$, the pruning is carried out from $t_{i-1}$ to $t_{i}-1$, and thus the pruned action set $Act^i$ and policy $\pi_C^i$ cover from $t_{i-1}$ to $t_{i}-1$. Therefore, $Act^i$ and $\pi_C^i$ for the subgraphs are disjoint and can combine into a single $Act(\cdot)$ and $\pi_C$ (lines 11-12). Note that the one-shot algorithm is essentially a special case of the multi-shot algorithm with the set of time $\{t_0=0, t_1 = T\}$ and set of probabilities $\{Pr_1 = Pr_{des}\}$.

In Alg.~2b, once the agent adopts the $\pi_C$ policy (line 7), it can switch to $Act$ if it reaches the accepting states or trash states of any subgraph (lines 15-16). Essentially, Alg.~2b applies Alg.~1b on each subgraph to learn a policy with a probabilistic constraint satisfaction guarantee. 

\begin{algorithm}[htb!]
 \label{alg:alg1}
 \begin{center}
\resizebox{\columnwidth}{!}{
\begin{tabular}{ll}
\bf{Alg. 2a:} \textbf{Multi-shot Pruning Algorithm for Time-total Product MDP} \\
\hline
 \emph{Input:} $ \mathcal{P}^{\mathcal{T}} = (S^{\mathcal{T}}_{\mathcal{P}}, P^{\mathcal{T}}_{init}, A, \Delta_\FP^{\mathcal{T}}, R^{\mathcal{T}}_{\mathcal{P}}, F^{\mathcal{T}}_{\mathcal{P}}, \Psi^{\mathcal{T}}_{\mathcal{P}})$ (time-total product MDP)\\
 \emph{Input:} $\{t_0=0, t_1, t_2, ..., t_{N-1}, t_N = T\}$, $\{Pr_1, Pr_2, ..., Pr_N\}$\\
\emph{Output:} $\mathcal{P}^{\mathcal{T}}$ (pruned time-total product MDP),\,$\pi_C$\\
\hline 

\mbox{\small $\;1:\;$}\hspace{0.1cm}$G_{i} \leftarrow \mathcal{P}^{\mathcal{T}}_{t_{i-1}:\,t_{i}}$, $i = N, N-1, ..., 1 $\\
\mbox{\small $\;2:\;$}\hspace{0.1cm}\textbf{Initialization:}  $T_s \leftarrow t_N,\,\, T_e \leftarrow t_{N-1}$, \,\,$Act(p^t)=A$ for all $p^t \in S^\mathcal{T}_\mathcal{P}$ \\
\hspace{2.5cm}  $f^N(p^{T_s}) \leftarrow \eqref{v(s)_T}$ for all $p^T \in S^\mathcal{T}_\mathcal{P}$ \\
\mbox{\small $\;3:\;$}\hspace{0.1cm}$f^N(\cdot), \,\pi_C^N, \,Act^N(\cdot) \leftarrow $ run lines 2-18 in Alg.~1a on $G_N$ with $Pr_N$\\
\mbox{\small $\;4:\;$}\hspace{0.1cm}$A_N \leftarrow F^{\mathcal{T}}_{\mathcal{P}}; T_N \leftarrow \Psi^{\mathcal{T}}_{\mathcal{P}} $\\
\mbox{\small $\;5:\;$}\hspace{0.1cm}\textbf{for} $i = N-1, N-2, ..., 1 $\\
\mbox{\small $\;6:\;$}\hspace{0.5cm} $A_i \leftarrow \{p^{t_i}|f^{i+1}(p^{t_i})\geq Pr_{i+1}\}$(set of accepting states for $G_i$)\\
\mbox{\small $\;7:\;$}\hspace{0.5cm} $T_i \leftarrow \{p^{t_i}|f^{i+1}(p^{t_i})< Pr_{i+1}\}$(set of trash states for $G_i$)\\
\mbox{\small $\;8:\;$}\hspace{0.5cm} $T_s \leftarrow t_i, T_e \leftarrow t_{i-1}$\\
\mbox{\small $\;9:\;$}\hspace{0.5cm} for $\forall p^{T_s}$ \textbf{if} $p^{t_i} \in A_i$ \textbf{then} $f^i(p^{t_i}) \leftarrow 1$; \textbf{else} $f^i(p^{t_i}) \leftarrow 0$  \\
\mbox{\small $\;10:\;$}\hspace{0.36cm} $G'_{i}, \,f^i(\cdot), \,\pi_C^i, \,Act^i(\cdot) \leftarrow $ run lines 2-18 in Alg.~1a on $G_i$ with $Pr_i$\\
\mbox{\small $\;11:\;$}$\pi_C \leftarrow$ combine $\pi_C^1, \pi_C^2, ..., \pi_C^N$\\
\mbox{\small $\;12:\;$}$Act(\cdot) \leftarrow$ combine $ Act^1, Act^2, ..., Act^N$\\
\mbox{\small $\;13:\;$}$\mathcal{P}^{\mathcal{T}} =(S^{\mathcal{T}}_{\mathcal{P}}, P^{\mathcal{T}}_{init}, Act:S_\FP^\mathcal{T} \rightarrow 2^A, \Delta_\FP^{\mathcal{T}}, R^{\mathcal{T}}_{\mathcal{P}}, F^{\mathcal{T}}_{\mathcal{P}}, \Psi^{\mathcal{T}}_{\mathcal{P}})$
\end{tabular}}
\end{center}
\end{algorithm}

\begin{algorithm}[htb!]
 \label{alg:alg1}
 \begin{center}
\resizebox{\columnwidth}{!}{
\begin{tabular}{ll}
\bf{Alg. 2b:} \textbf{Q-Learning with Constraint Satisfaction Guarantee} \\
\hline
 \emph{Input:} $\mathcal{P}^{\mathcal{T}} =(S^{\mathcal{T}}_{\mathcal{P}}, P^{\mathcal{T}}_{init}, Act:S_\FP^\mathcal{T} \rightarrow 2^A, \Delta_\FP^{\mathcal{T}}, R^{\mathcal{T}}_{\mathcal{P}}, F^{\mathcal{T}}_{\mathcal{P}}, \Psi^{\mathcal{T}}_{\mathcal{P}})$ \\
  \emph{Input:} $p_{init}=(s_{init},q_{init},0) \in P_{init}^{\mathcal{T}}$, $\pi_C$, $A_{i(i=1,2,...N)}$, $T_{i(i=1,2,...N)}$ \\
\emph{Output:}  $\pi:S^\mathcal{T}_\mathcal{P} \rightarrow A$ \\
\hline 
\mbox{\small $\;1:\;$}\textbf{Initialization:} Initial $Q-$table, $p \leftarrow p_{init}$;  \\
\mbox{\small $\;2:\;$}$\text{flag}^{\pi_C} \leftarrow $ False; \\
\mbox{\small $\;3:\;$}\hspace{0.05cm}\textbf{for}\hspace{0.1cm}  $j=0:N_{episode}$ \\
\mbox{\small $\;4:\;$}\hspace{0.5cm}\textbf{for} $t=0:T$ \\ 
\mbox{\small $\;5:\;$}\hspace{0.9cm}\textbf{if} $\text{flag}^{\pi_C}$ or $Act(p)=\emptyset$\\
\mbox{\small $\;6:\;$}\hspace{1.25cm} $a= \pi_C(p)$;\\
\mbox{\small $\;7:\;$}\hspace{1.25cm} $\text{flag}^{\pi_C} \leftarrow $ True;\\
\mbox{\small $\;8:\;$}\hspace{0.9cm}\textbf{else}\\
\mbox{\small $\;9:\;$}\hspace{1.4cm}Select an action \emph{a} from $Act(p)$ via $\epsilon-$greedy ($\pi$ policy);\\
\mbox{\small $\;11:\;$}\hspace{0.75cm}Take action \emph{a}, observe the next state $p^\prime=(s^\prime,q^\prime, t+1)$ and reward \emph{r};\\
\mbox{\small $\;12:\;$}\hspace{0.75cm}$Q(p,a) = (1-\alpha_{ep}) Q(p,a) + \alpha_{ep} \big[ r + \gamma \max\limits_{a^\prime}  Q(p^\prime,a^{\prime}) \big]$; \\
\mbox{\small $\;13:\;$}\hspace{0.75cm}$\pi(p) = \arg\max\limits_a Q(p,a))$; \\
\mbox{\small $\;14:\;$}\hspace{0.75cm}$p = p^\prime$; \\
\mbox{\small $\;15:\;$}\hspace{0.75cm}\textbf{if} $p$ in $A_i$ or $T_i$\\
\mbox{\small $\;16:\;$}\hspace{1.1cm} $\text{flag}^{\pi_C} \leftarrow $ False;\\
\mbox{\small $\;17:\;$}\hspace{0.35cm}\textbf{end for}\\
\mbox{\small $\;18:\;$}\hspace{0.35cm}$ p = (s^\prime,q_{init},0);$ \\
\mbox{\small $\;19:\;$}\textbf{end for} 
\end{tabular} }
\end{center}
\end{algorithm}
\vspace{-2mm}

\begin{theorem}
\label{theorem2}
Given a time-total product MDP $\mathcal{P}^{\mathcal{T}}$, a set of time $\{t_0=0, t_1, t_2, ..., t_{N-1}, t_N = T\}$, and a set of desired probabilities $\{Pr_1, Pr_2, ..., Pr_N\}$ such that $\prod\limits_{i=1}^{N}Pr_i = Pr_{des}$, if the following conditions hold in Alg.~2a: 
\begin{itemize}
   \item any initial state of the time-total product MDP $p^0_i \in P^{\mathcal{T}}_{init}$ satisfies $f^1(p^0_i) \geq Pr_{1}$,
\item  the set of accepting states $A_i$ defined in Alg.~2a (line 5) is nonempty for every $i\in \{1, \hdots, N-1\} $,
\end{itemize}
then in Alg.~2b, the probability of reaching the accepting states of $\mathcal{P}^{\mathcal{T}}$ in every episode is at least $Pr_{des}$.
\end{theorem}
\begin{proof}
Alg.~2b divides the time-total product MDP $\mathcal{P}^{\mathcal{T}}$ into $N$ subgraphs and applies Alg.~1b on each subgraph $G_i$ in a concatenated manner. By condition (a) and Theorem \ref{theorem1}, the probability of reaching the set of accepting states $A_1$ on subgraph $G_1$ from any initial state is at least $Pr_1$. For any $i=1,2,...,N-1$, since $A_i$ is nonempty (condition (b)), each state $p^{t_i} \in A_i$ satisfies $f^{i+1}(p^{t_i})\geq Pr_{i+1}$ (line 5 in Alg.~2a). By Theorem \ref{theorem1}, starting from any state in $A_i$, the probability of reaching the set of accepting states $A_{i+1}$ on subgraph $G_{i+1}$ is at least $Pr_{i+1}$. Then, starting from any initial state of $\mathcal{P}^{\mathcal{T}}$, the probability of reaching $A_N$, i.e., the set of accepting states of $\mathcal{P}^{\mathcal{T}}$ is at least $\prod\limits_{i=1}^{N}Pr_i = Pr_{des}$.
\end{proof}

\begin{table*}[ht!]
\centering
\begin{adjustbox}{width=0.93\textwidth}
\begin{tabular}{||c c c c c c c c c c c||}
 \hline
 Algorithm & ($\epsilon, Pr_{des}$) & (0.03, 0.5) & (0.03, 0.7) & (0.03, 0.9)  & (0.08, 0.5) & (0.08, 0.7) & (0.08, 0.9) & (0.13, 0.5) & (0.13, 0.7) & (0.13, 0.9) \\ [0.5ex] 
 \hline\hline
 \multirow{3}{*}{One-shot} &
 \multicolumn{1}{l}{Learning} & 
 \cellcolor{gray!20}\textcolor{blue}{91.66\%} & 
 \cellcolor{gray!45}\textcolor{blue}{91.63\%} & \cellcolor{gray!60}\textcolor{blue}{91.64\%} & \cellcolor{gray!20}\textcolor{blue}{91.60\%} & \cellcolor{gray!45}\textcolor{blue}{91.61\%} & \cellcolor{gray!60}\textcolor{blue}{99.48\%} & \cellcolor{gray!20}\textcolor{blue}{91.68\%} & \cellcolor{gray!45}\textcolor{blue}{99.32\%} & \cellcolor{gray!60}\textcolor{blue}{99.53\%} \\ &
 \multicolumn{1}{l}{Testing} & \cellcolor{gray!20}\textcolor{blue}{90.92\%} & \cellcolor{gray!45}\textcolor{blue}{91.08\%} & \cellcolor{gray!60}\textcolor{blue}{91.77\%} & \cellcolor{gray!20}\textcolor{blue}{91.43\%} & \cellcolor{gray!45}\textcolor{blue}{91.5\%} & \cellcolor{gray!60}\textcolor{blue}{99.39\%} & \cellcolor{gray!20}\textcolor{blue}{91.55\%} & \cellcolor{gray!45}\textcolor{blue}{99.53\%} & \cellcolor{gray!60}\textcolor{blue}{99.67\%} \\ &
 \multicolumn{1}{l}{Avg. Rewards} & \cellcolor{gray!20}\textcolor{blue}{20.99} & \cellcolor{gray!45}\textcolor{blue}{20.97} & \cellcolor{gray!60}\textcolor{blue}{20.88} & \cellcolor{gray!20}\textcolor{blue}{20.93} & \cellcolor{gray!45}\textcolor{blue}{20.90} & \cellcolor{gray!60}\textcolor{blue}{19.72} & \cellcolor{gray!20}\textcolor{blue}{20.91} & \cellcolor{gray!45}\textcolor{blue}{19.73} & \cellcolor{gray!60}\textcolor{blue}{19.70}
 \\\hline\hline
 \multirow{3}{*}{Multi-shot} &
 \multicolumn{1}{l}{Learning} & 
 \cellcolor{gray!20}\textcolor{red}{78.93\%} & 
 \cellcolor{gray!45}\textcolor{red}{93.52\%} & \cellcolor{gray!60}\textcolor{red}{97.96\%} & \cellcolor{gray!20}\textcolor{red}{93.18\%} & \cellcolor{gray!45}\textcolor{red}{97.61\%} & \cellcolor{gray!60}\textcolor{red}{99.65\%} & \cellcolor{gray!20}\textcolor{red}{97.54\%} & \cellcolor{gray!45}\textcolor{red}{99.22\%} & \cellcolor{gray!60}\textcolor{red}{99.93\%} \\ &
 \multicolumn{1}{l}{Testing} & \cellcolor{gray!20}\textcolor{red}{77.4\%} & \cellcolor{gray!45}\textcolor{red}{92.34\%} & \cellcolor{gray!60}\textcolor{red}{97.8\%}& \cellcolor{gray!20}\textcolor{red}{97.04\%} & \cellcolor{gray!45}\textcolor{red}{99.11\%} & \cellcolor{gray!60}\textcolor{red}{99.61\%} & \cellcolor{gray!20}\textcolor{red}{99.15\%} & \cellcolor{gray!45}\textcolor{red}{99.33\%}& \cellcolor{gray!60}\textcolor{red}{99.89\%} \\ 
 &
 \multicolumn{1}{l}{Avg. Rewards} & \cellcolor{gray!20}\textcolor{red}{79.03} & \cellcolor{gray!45}\textcolor{red}{55.98} & \cellcolor{gray!60}\textcolor{red}{49.84} & \cellcolor{gray!20}\textcolor{red}{49.29} & \cellcolor{gray!45}\textcolor{red}{39.65} & \cellcolor{gray!60}\textcolor{red}{38.84} & \cellcolor{gray!20}\cellcolor{gray!20}\textcolor{red}{40.55} &\cellcolor{gray!45}\textcolor{red}{39.62} & \cellcolor{gray!60}\textcolor{red}{20.74}
 \\
 \hline
\end{tabular}
\end{adjustbox}
\caption{Simulation results for
the task
$[H^1 P]^{[0,8]} \cdot ([H^1 D_1]^{[0,6]})\cdot ([H^1 D_2]^{[0,6]}) \; \vee \; [H^1 D_3]^{[0,6]}) \cdot [H^1 Base]^{[0,12]}$ and the real action uncertainty of $\epsilon_{real}=0.03$ 
for the one-shot and the multi-shot learning algorithms. The first two rows for each algorithm present the satisfaction ratio of the constraint over 1000000 learning episodes and 10000 testing episodes respectively. The third row for each algorithm shows the average episodic rewards for 10000 episodes during testing.}
\label{table:result}
\end{table*}

\section{Simulation Results}
We present a case study where a robot operates on a 6$\times$6 grid environment with an action set $A = \{N,NE,E,SE,S,SW,W,NW,Stay\}$ as shown in Fig.~\ref{fig:case}. Under these actions, the agent can either move to any of the feasible adjacent cells in the 8 directions or stay at its current location. The transition probability (unknown to the agent) is as follows: Each of the first 8 actions leads to the intended transition with a probability of 0.97 or to one of the other feasible transitions (selected uniformly at random) with a probability of 0.03. For example, if the agent takes the action $E$, it will move to the adjacent cell in the direction $E$ with a probability of 0.97 (intended transition) or, with a probability of 0.03, it will stay at its location or move to a feasible adjacent cell in any of the other 7 directions (unintended transitions). If the agent takes the action $Stay$, it will stay at the original position with probability 1. In this case study, the lower bound and upper bound in \eqref{not_really_an_assumption} is given by introducing an overestimated transition uncertainty $\epsilon \geq 0.03$. More specifically, the available prior information is that each action yields the corresponding intended transition with some probability in $[1-\epsilon, 1]$. Furthermore, each of the unintended transitions has a probability in $[0,\epsilon]$. 

We consider a scenario where the agent is required to periodically perform a pickup and delivery task while maximizing situational awareness by collecting measurements from the environment. The pickup and delivery task is encoded as a TWTL\footnote{The syntax and semantics of TWTL can be found in the Appendix.} constraint: $[H^1 P]^{[0,8]} \cdot ([H^1 D_1]^{[0,6]})\cdot ([H^1 D_2]^{[0,6]}) \; \vee \; [H^1 D_3]^{[0,6]}) \cdot [H^1 Base]^{[0,12]}$, which means that \textit{``go to the pickup location $P$ and stay there for $1$ time step in the first $8$ time steps and immediately after that go to $D_1$ and stay there for $1$ time step within $6$ time steps, and immediately after that go to either $D_2$ or $D_3$ within $6$ time steps and stay there for $1$ time step, and immediately after that go to $Base$ and stay there for $1$ time step within $12$ time steps."}. Based on the time bound of this TWTL specification, the length of each episode is selected as $35$ time steps. 

In Fig.~\ref{fig:case}, we present two sample trajectories by applying the learned policies obtained by the one-shot and multi-shot algorithms, with $\epsilon = 0.08$ and $Pr_{des}=0.9$. Following the one-shot policy, the robot adopts the $\pi_C$ policy at $t=4$ until it satisfies the TWTL constraint. Though the one-shot algorithm provides a constraint satisfaction guarantee, following the $\pi_C$ policy until satisfying the constraints prevents the agent from learning a policy to visit the high-reward location. The multi-shot algorithm relaxes such a requirement by allowing the agent to explore before satisfying the constraints (lines 15-16 in Alg.~2b) and can potentially learn a policy yielding higher reward than the one-shot algorithm.

\begin{figure}[ht]
    \centering
    \hspace*{-4cm}
    \subfigure[]{\includegraphics[width=0.16\textwidth]{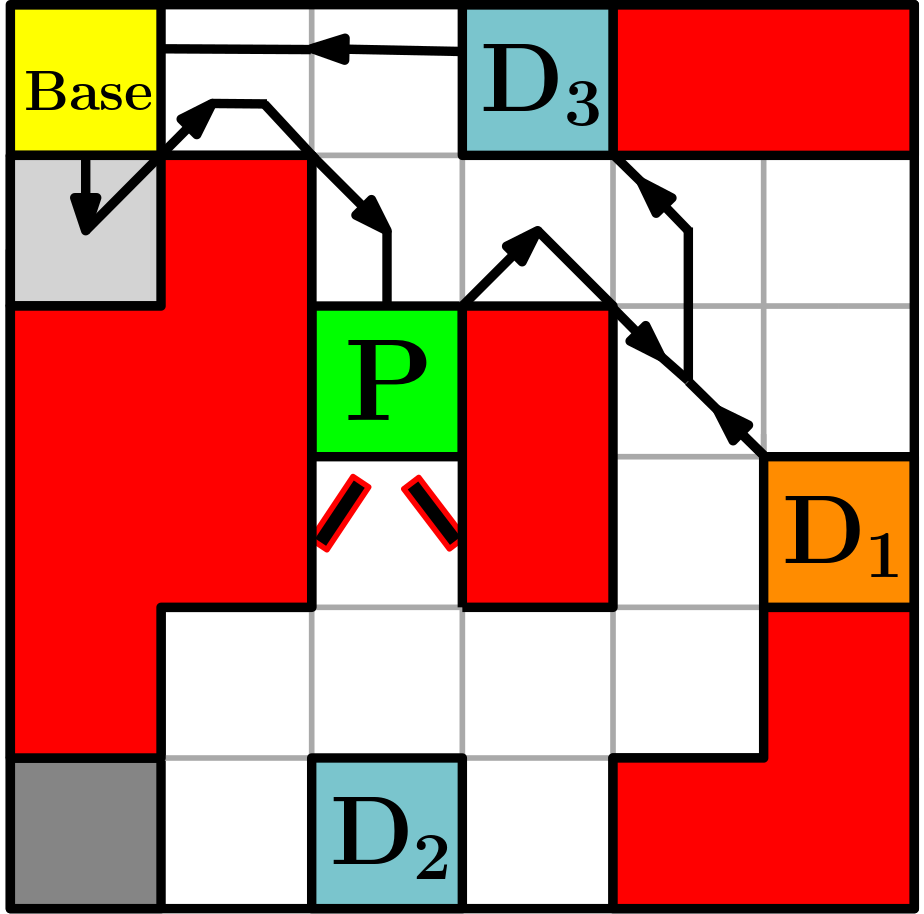}}
    \hspace{1cm}
    \subfigure[]{\includegraphics[width=0.16\textwidth]{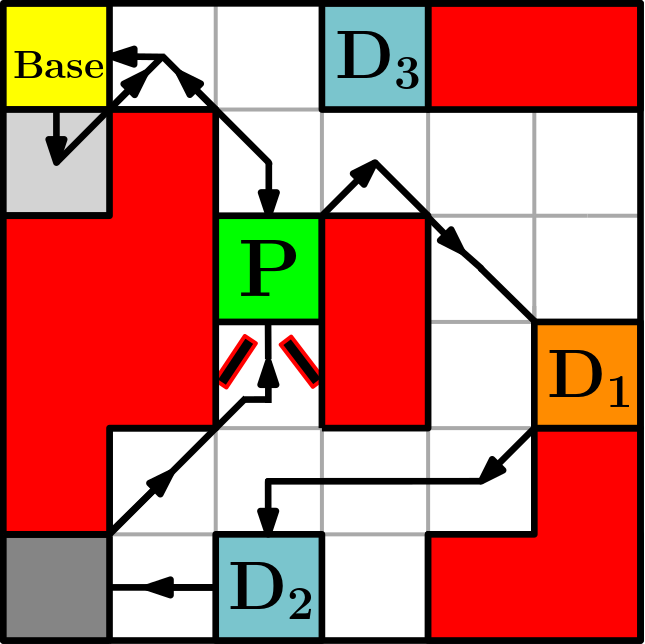}}
    \hspace*{-4cm}
    \caption{The gray cells (darker shade indicates higher reward) are the regions where monitoring is rewarded. The black bars under the P location represent a one-way door that only allows the "North" action. The black arrows denote sample trajectories obtained from the two algorithms with $(\epsilon, Pr_{des}) = (0.08, 0.9)$: (a) one-shot and (b) multi-shot.}
    \label{fig:case}
\end{figure}

We investigate how the parameters $\epsilon$ and $Pr_{des}$ influence the performance of the one-shot and multi-shot algorithms. The proposed algorithms are executed with varying values of $Pr_{des} =0.5,0.7, 0.9 $ and $\epsilon =0.03,0.08,0.13$. Since the constraint involves 4 sub-tasks, for the multi-shot algorithm, the set of time for dividing the time-total product MDP is chosen as $\{t_0=0, t_1 = 8, t_2 = 15,t_3 = 22, t_4 = 35\}$ where $t_1$, $t_2$, $t_3$, and $t_4$ are the deadline by which the sub-tasks must be completed. The set of desired probabilities for the multi-shot algorithm is selected as $\{\sqrt[\uproot{3}4]{Pr_{des}}, \sqrt[\uproot{3}4]{Pr_{des}}, \sqrt[\uproot{3}4]{Pr_{des}}, \sqrt[\uproot{3}4]{Pr_{des}} \}$.
Note that the selection of the set of time steps and the set of desired probabilities is not unique, as long as the conditions in Theorem 2 hold. The results are shown in Table \ref{table:result}.

 The performance of the multi-shot algorithm is significantly influenced by the parameters $\epsilon$ and $Pr_{des}$. For a fixed $Pr_{des}$, we observe that the satisfaction rate increases with $\epsilon$, while the average rewards decrease with $\epsilon$. Since we use the transition uncertainty $\epsilon$ to determine the lower and upper bounds of the transition probability of the system and solve for the worst-case maximum satisfaction probability, we will obtain lower probability values with a higher $\epsilon$. Therefore, more actions will be pruned (see Alg.~1a line 6) and the RL agent is more likely to adopt the $\pi_C$ policy which maximizes the satisfaction probability. The more confined action set resulting from higher $\epsilon$ also explains the decrease in average rewards. Similarly, for a fixed $\epsilon$, we notice that a higher $Pr_{des}$ results in a higher satisfaction rate and lower rewards. A higher $Pr_{des}$ implies pruning more actions (see Alg.~1a line 6), which also explains the increase in satisfaction rate and decrease in average rewards. 
 
 The $\epsilon$ and $Pr_{des}$ parameters have no significant influence on the average rewards and satisfaction rate in the one-shot algorithm. This is due to the fact that the agent has to adopt the $\pi_C$ policy in the earlier stage of the episodes regardless of the $Pr_{des}$ and $\epsilon$ parameters. Otherwise, it will not be able to reach the P location before $t=8$ and will fail the constraints. As a result, the agent adopts the $\pi_C$ policy in the early stage until it satisfies or fails the TWTL constraint. Therefore, the one-shot algorithm learns a similar policy under different parameters (as in Fig.~\ref{fig:case}a).

\section{Conclusion}

We presented a constrained RL algorithm for keeping the probability of satisfying a bounded TL constraint above a desired threshold throughout learning. The proposed method is based on integrating the total automaton representation of the constraint into the underlying MDP and avoiding ``unsafe" actions based on some available prior information given as upper and lower bounds for each transition probability.  We theoretically showed that, under some conditions on the desired probability threshold and the MDP, the proposed approach ensures the desired probabilistic constraint satisfaction throughout learning. We also provided numerical results to demonstrate the performance of our proposed approach. As a future research, we plan to investigate the optimal implementation of the multi-shot algorithm; i.e., the optimal decomposition of the time and probability sets.

\bibliographystyle{IEEEtran}
\bibliography{refer}

\appendix
\noindent \textit{Time-Window Temporal Logic \cite{twtl}:}
Let $AP$ be a set of atomic propositions, each of which has a truth value over the state-space. A TWTL formula is defined over the set $AP$ with the following syntax:
\vspace{2mm}

\centerline{$\phi ::= H^d x | H^d \notltl x | \phi_1 \andltl \phi_2 | \phi_1 \orltl \phi_2  | \notltl \phi_1 |  \phi_1 . \phi_2 | [\phi_1]^{[a,b]}$, where}
\begin{itemize}
\item $x$ is either the true constant $\top$ or an atomic proposition from $AP$;
\item $\andltl$, $\orltl$, and $\notltl$ are the conjunction, disjunction, and negation Boolean operators, respectively;
\item $\cdot$ is the concatenation operator;
\item $H^d$ with d $\in$ $\mathbb{Z}_{\geq 0}$ is the hold operator;
\item $[]^{[a,b]}$ with $0 \leq a \leq b$ is the within operator. 
\end{itemize}

\vspace{1mm}
The semantics of TWTL is defined according to the finite words $\mathbf{o}$ over $AP$, and $o(k)$ refers to the $k^{th}$ element on $\mathbf{o}$. For any $x \in AP$,  
the {\em hold} operator $H^d x$ indicates that $x$ should be true (serviced) for $d$ time units (i.e., $o \models H^d x$ if $o(t)=x \;\; \forall t\in [0,d]$).
The {\em within} operator $[\phi]^{[a, b]}$ means that the satisfaction of $\phi$ is bounded to the time window $[a, b]$ (i.e., $o \models [\phi]^{[a, b]}$ if $\exists k \in (0,b-a) \text{ s.t. } o^{\prime} \models \phi$ where $o^{\prime} = o(a+k)\dots o(b)$).
Finally, the concatenation of $\phi_i$ and $\phi_j$ (i.e., $\phi_i \cdot \phi_j$) designates that the first $\phi_i$ must be satisfied and then immediately after that
$\phi_j$ must be satisfied.
\end{document}